\documentclass[a4paper,12pt]{article}
\usepackage{mdl}
\usepackage{bm}
\usepackage{amsthm}
\usepackage[ruled,linesnumbered]{algorithm2e}
\usepackage{url}
\usepackage{hyperref}
\usepackage{autonum}
\usepackage{amsmath}
\newtheorem{definition}{Definition}
\newtheorem{theorem}{Teorem}
\newtheorem{lemma}{Lemma}
\newtheorem{proposition}{Proposition}

\title{Differentially Private Analysis of Outliers}
\author{
Rina Okada, Kazuto Fukuchi, Kazuya Kakizaki\\
\texttt{\{rina, kazuto, kazuya\}@mdl.cs.tsukuba.ac.jp}\\
Graduate School of SIE, University of Tsukuba\\
\and
Jun Sakuma\\
\texttt{jun@cs.tsukuba.ac.jp}\\
Graduate School of SIE, University of Tsukuba / JST CREST
}
\makeatletter
\def\@maketitle{%
 \newpage
 \null
 \vskip 2em%
 \begin{center}%
 \let \footnote \thanks
   {\LARGE \@title \par}%
   \vskip 1.5em%
   {\large
     \begin{tabular}[t]{c}%
       \@author
     \end{tabular}\par}%
 \end{center}%
 \par
 \vskip 1.5em}
\makeatother

\begin{document}
\maketitle

\abstract{
This paper investigates differentially private analysis of distance-based outliers. 
The problem of outlier detection is to find a small number of instances that are apparently distant from the remaining instances. On the other hand, the objective of differential privacy is to conceal presence (or absence) of any particular instance. Outlier detection and privacy protection are thus intrinsically conflicting tasks. 
In this paper, instead of reporting outliers detected, we present two types of differentially private queries that help to understand behavior of outliers.
One is the query to count outliers, which reports the number of outliers that appear in a given subspace. Our formal analysis on the exact global sensitivity of outlier counts reveals that regular global sensitivity based method can make the outputs too noisy, particularly when the dimensionality of the given subspace is high. Noting that the counts of outliers are typically expected to be relatively small compared to the number of data, we introduce a mechanism based on the smooth upper bound of the local sensitivity. 
The other is the query to discovery top-$h$ subspaces containing a large number of outliers. This task can be naively achieved by issuing count queries to each subspace in turn. However, the variation of subspaces can grow exponentially in the data dimensionality. This can cause serious consumption of the privacy budget. For this task, we propose an exponential mechanism with a customized score function for subspace discovery.
To the best of our knowledge, this study is the first trial to ensure differential privacy for distance-based outlier analysis. We demonstrated our methods with synthesized datasets and real datasets. The experimental results show that out method achieve better utility compared to the global sensitivity based methods.
}

\section{Introduction}
Machine learning and data mining technologies are now becoming increasingly influential in our daily life. When data mining is processed over personal data collected from individuals, the acquired knowledge might be used to infer private information. Differential privacy is a recent notion of privacy tailored to the problem of releasing statistical information~\cite{Dwork:2006:CNS:2180286.2180305}.  Differential privacy for statistical queries of various types, such as average, sum, variance, histogram, median, and maximum likelihood estimator, have been investigated \cite{Dwork:2006:CNS:2180286.2180305,Nissim:2007:SSS:1250790.1250803,dwork2010differential}.

As described in this paper, we investigate differentially private outlier analysis. Outlier detection is a task to identify instances that are apparently distant from the remaining instances. The objective of differential privacy is to prevent adversaries from learning of the presence (or absence) of any particular instance from released information. Outlier detection and privacy protection are therefore intrinsically conflicting tasks. It presents a challenging difficulty.

To overcome this difficulty, instead of identifying outliers, we consider reporting statistical aggregation on outliers that helps to recognize the occurrence of anomalous situations, with a guarantee of differential privacy. More specifically, we examine differentially private queries of three types for outlier analysis. One is a query to count outliers that appear in a given subspace. Second is a query to discover the top-$h$ subspaces containing numerous outliers. Third is a query to detect the top-$h$ outliers are that more likely.

\subsection{Related Works}
\label{subsec:relatedWorks}
We introduce existing studies of privacy aspects of outlier analysis.
Secure multiparty computation (SMC) is a cryptographic tool that facilitates the evaluation of a specified function over their private inputs jointly, while maintaining these inputs as private. 
 One earlier study\cite{Outlier04} introduced an SMC for distance-based outlier detection from horizontally and vertically partitioned private databases using random shares. One earlier study \cite{xue2008privacy} investigated an SMC for spatial outlier detection. Another report of a study \cite{dung2011distributed} presented an SMC for distance-based outlier detection with the Mahalanobis distance. Another study \cite{li2013privacy} presented an SMC for density-based outlier detection. The objective of these works is to detect outliers securely without mutually sharing privately distributed data; privacy invasion caused by observing detected outliers is not considered. 

Studies of differential privacy for outlier analysis are few, presumably because of its intrinsic difficulty, as described. Only one report in the literature \cite{DBLP:conf/icdm/FanX13} describes a study that considers the differential privacy of outlier analysis. This study was conducted to detect anomalous changes from a time series under a guarantee of differential privacy. 
The objective of this study is closely related to ours, whereas this method releases a one-dimensional time series with differential privacy; outlier detection is applied to the released data as a post process. Consequently, the approach differs from ours.

\cite{lui2014outlier} introduced a novel privacy notion, outlier privacy, as a generalization of differential privacy. Outlier privacy measures an individual's privacy parameter by how much of an ``outlier'' the individual is. The objective of this study is to define privacy using the notion of outliers, but not for differentially private outlier analysis.

\subsection{Our Contribution}
  \label{subsec:contribution}
  In this paper, we present a methodology for distance-based outlier analysis with guarantee of differential privacy. Our proposal consists of two different types of differentially private queries.

  {\bf Differentially private counting of outliers.} This query reports the number of outliers that appear in a given subspace. Since the global sensitivity of counts of outliers is very large, the resulting outputs can be too noisy. We focus on the observation that the counts of outliers are expected to be relatively small compared to the number of data in typical datasets. Taking advantage of this, we develop a randomization mechanism for counts of outliers based on the smooth upper bound of the local sensitivity~\cite{Nissim:2007:SSS:1250790.1250803}. Randomization mechanism based on the smooth upper bound typically have better utility because of its data dependency; however, its evaluation is often costly. To alleviate this, we provide an efficient algorithm for evaluation of smooth upper bound for counting outliers. 

  {\bf Differentially private discovery of subspaces.}  This query finds top-$h$ subspaces containing a large number of outliers. This task can be naively achieved by issuing count queries to each subspace in turn. However, the variation of subspaces can grow exponentially in the data dimensionality. This can cause serious consumption of the privacy budget. For this task, we employed the exponential mechanism. We specifically design a score function for subspace discovery which is insensitive to the size of the subspace set. Because of this insensitivity, the proposed mechanism achieves better detection accuracy even with high dimensionality. 

  To the best of our knowledge, this study is the first trial to ensure differential privacy for distance-based outlier analysis. 
We demonstrated our methods with synthesized datasets and real datasets. The experimental results show that our methods achieve better utility compared to the global sensitivity based methods. 

\section{Differential Privacy}
 \label{sec:DP}
 Let $X=\{\bm{x}_1, \bm{x}_2, \hdots, \bm{x}_N\} \in \mathcal{D}^N$ be a database.
 An {\it analyst} issues a query $f: \mathcal{D}^N \to \mathcal{T}$; then the database returns an output, where $\mathcal{T}$ denotes the range of the outputs.
 {\em Differential privacy}, a recent notion of privacy, measures the privacy breach of database $X$ caused by releasing output $t \in \mathcal{T}$ with no assumptions of the background knowledge of adversaries.
 The outputs are typically modified using a {\em mechanism} $\mathcal{A}:\mathcal{D}^N \to \mathcal{T}$ before release to preserve differential privacy.
 
 Let $H(X, X')=|\{i:\bm{x}_i \neq \bm{x}'_i\}|$ denote the Hamming distance, the number of different records in $X$ and $X'$. If $H(X,X')=1$, then it can be said that $X$ and $X'$ are neighbor databases, or $X \sim X'$ shortly. In the following, we presume $|X|=|X'|=N$. Then, $(\epsilon,\delta)$-differential privacy is defined as shown below.
 
 \begin{definition}[$(\epsilon,\delta)$-Differential Privacy]  \label{DP}
  Mechanism $\mathcal{A}$ guarantees $(\epsilon,\delta)$-differential privacy if, $\forall X \sim X'$ and $\forall T \subseteq \mathcal{T}$,
  \begin{equation}  \label{DP-eq}
   Pr[\mathcal{A}(X) \in T] \leq e^ \epsilon Pr[\mathcal{A}(X') \in T] + \delta.
  \end{equation}
 \end{definition}
 The parameter $\epsilon$ and $\delta$ are designated as privacy parameters.
Randomization based on the global sensitivity is the most straightforward realization of differential privacy for continuous outputs~\cite{Dwork:2006:CNS:2180286.2180305}. The exponential mechanism is a natural extension for discrete outputs~\cite{conf/eurocrypt/DworkKMMN06}.
We use both mechanisms for our method, which is explained in detail in the next subsection.

\subsection{Sensitivity-based Method}
\subsubsection{Global Sensitivity}
Presuming that the output domain of query $f$ is continuous, then randomization based on the global sensitivity \cite{Dwork:2006:CNS:2180286.2180305} provides a mechanism that guarantees differential privacy for queries of any type, as long as its global sensitivity is evaluable.
The global sensitivity is defined as explained below.
  \begin{definition}[Global Sensitivity]
   \label{GS-defi}
Letting $\mathcal{D}$ be the domain of data, the global sensitivity of query $q:\mathcal{D}^N \to \mathbb{R}^d$ is given as
   \begin{equation}
    \label{GS-defi-eq}
     GS_{q} = \max_{X \sim X'} \|q(X)-q(X')\|_2,
   \end{equation}
where $\| \cdot \|_2$ denotes $\ell_2$ norm of vectors.
  \end{definition}
Given the global sensitivity for a specified query, randomization by a normal distribution based on the global sensitivity guarantees $(\epsilon,\delta)$-differential privacy, as stated by the following theorem.
  \begin{theorem}[Gaussian Mechanism by Global Sensitivity \cite{conf/eurocrypt/DworkKMMN06}]
   \label{LM}
Let $GS_{q}$ be the global sensitivity of a query $q: \mathcal{D}^N \to \mathbb{R}^d$. Then, mechanism $\mathcal{A}$ that randomizes the output of the query by eq. (\ref{LM-eq}) provides $(\epsilon,\delta)$-differential privacy
   \begin{equation}
    \label{LM-eq}
     \mathcal{A}_q(X) = q(X) + Y,
   \end{equation}
where $Y \in \mathbb{R}$ denotes a noise in which $Y$ is an sample drawn from the Gaussian distribution with mean $0$ and variance $\frac{GS_q^2 \cdot 2\log{(2/\delta)}}{\epsilon^2}$.
  \end{theorem}

\subsubsection{Smooth Sensitivity}
\label{subsubsec:smooth}
For some functions, the global sensitivity can be impractically large even when the sensitivities are small with almost all neighboring pairs. This large sensitivity occurs because it is evaluated as the greatest difference of outputs among many possible neighboring pair of databases.
For example, the global sensitivity of median is $N$, the whole sample size, but this arises only in a pathological situation~\cite{Nissim:2007:SSS:1250790.1250803}.
Randomization base on the smooth sensitivity enables the use of moderate sensitivity for such sensitive queries.
For a given database $X$, the {\em local sensitivity} is defined as the greatest difference of outputs for $\forall X'$ s.t. $X' \sim X$.

\begin{definition}[Local Sensitivity]
 \label{LS-defi}
 Let $\mathcal{D}$ be the domain of the data. Then, the local sensitivity of query $q:\mathcal{D}^N\to\mathbb{R}^d$ is given as
 \begin{equation}
  \label{LS-defi-eq}
   LS_q(X) = \max_{X':H(X,X')=1} ||q(X)-q(X')||.
 \end{equation}
\end{definition}
It is noteworthy that that $GS_q = \max_{X \in \mathcal{D}^N} LS_q(X)$.
Nissim et al.~\cite{Nissim:2007:SSS:1250790.1250803} presented the {\em smoothed sensitivity}, which is a class of smooth upper bounds to the local sensitivity.
\begin{definition}[Smooth upper bound] \label{Sub-defi}
For $\beta>0$, a function $S_\beta:D^n \rightarrow \mathbb{R}^+$ is a $\beta$-smooth upper bound on the local sensitivity of query $q$ if it satisfies the following requirements:
\begin{align}    \label{Sub-defi-eq}
\forall X \in D^n&,& S_{q,\beta}(X) \ge LS_f(X); \\
\forall X \sim X'&,& S_{q,\beta}(X) \le e^\beta S_{q,\beta}(X').
\end{align}
\end{definition}
The smallest function satisfying Definition \ref{Sub-defi} is the smooth sensitivity of $q$:
\begin{definition}[Smooth Sensitivity] \label{S-defi}
Given $\beta > 0$, the smooth sensitivity of query $q:\mathcal{D}^N \to \mathbb{R}^p$ is
\begin{equation}    \label{S-defi-eq}
     S^{*}_{q, \beta}(X) = \max_{X' \in \mathcal{D}^N} ( LS_{q}(X') \cdot e^{-\beta H(X,X')}).
\end{equation}
\end{definition}
\cite{Nissim:2007:SSS:1250790.1250803} also showed that adding noise proportional to the smooth sensitivity yields a private output perturbation mechanism if the noise distribution satisfies some properties. The differential privacy of a Gaussian mechanism realized by the smooth sensitivity can be stated by the following theorem.
 \begin{theorem}[Gaussian Mechanism by Smooth Sensitivity \cite{Nissim:2007:SSS:1250790.1250803}]   \label{M-S}
Let $Y$ be a noise generated from the Gaussian distribution with mean $0$ and variance $1$.
Let $S_{q,\beta}$ be a $\beta$-smooth upper bound of query $q$. Then, if $\alpha = \frac{\epsilon}{5\sqrt{2\ln{2/\delta}}}$ and $\beta = \frac{\epsilon}{4(p+\ln{2/\delta})}$, mechanism $\cal{A}$ guarantees $(\epsilon,\delta)$-differential privacy:
   \begin{equation}
    \label{M-S-eq}
    \mathcal{A}_q(X) = q(X) + \frac{S_{q, \beta}(X)}{\alpha} \cdot Y.
   \end{equation}
  \end{theorem}

 \subsection{Exponential Mechanism}
  The sensitivity-based method basically presumes that outputs of the target query are real-valued. The exponential mechanism is a natural extension of the sensitivity-based method to discrete outputs.
  Intuitively, the exponential mechanism relies on a utility function $u: \mathcal{D}^N \times \mathcal{R} \to \mathbb{R}$ that outputs a larger value if the input to the utility function is close to the true output. With this utility function, output values that are closer to the true output are likely to be provided by the exponential mechanisms, and vice versa. The sensitivity of the utility function is defined as presented below.
 \begin{definition}[Sensitivity of a utility function $u$]
  \label{delta}
  Let $\mathcal{D}$ be the domain of data, and let $u: \mathcal{D}^N \times \mathcal{R} \to \mathbb{R}$ be a utility function. Then, the sensitivity of $u$ is given as
  \begin{equation}
   \label{delta-eq}
    \Delta u =  \max_{r \in \mathcal{R}} \max_{X,X' \in \mathcal{D}^N:H(X,X')=1} ||u(X,r)-u(X',r)||.
  \end{equation}
 \end{definition}
 Given the sensitivity of a utility function, randomization of outputs following Theorem \ref{EM} guarantees $\epsilon$-differential privacy.
  \label{S-EM}
  \begin{theorem}[Exponential Mechanism \cite{McSherry:2007:MDV:1333875.1334185}]
   \label{EM}
   Let $\Delta u$ be the sensitivity of utility function $u:\mathcal{D}^{N} \times \mathcal{R} \to \mathbb{R}$. Then, mechanism $\epsilon^{\epsilon}_{u,\Delta u}$ that randomizes the output of the query by eq. (\ref{eq:expo-mecha}) provides $2\epsilon \Delta u$-differential privacy
   \begin{equation}
    Pr[\epsilon^{\epsilon}_{u,\Delta u}(X) = t \in \mathcal{R}] = \frac{\exp(\epsilon \cdot u(X,t))}{\int_{\mathcal{R}} \exp(\epsilon \cdot u(X,r))dr}. \label{eq:expo-mecha}
   \end{equation}
   \end{theorem}

\section{Problem Statement}
\label{sec:problemStatement}
 Our objective is to analyze outliers contained in a private database in a differentially private manner. Outlier detection is a problem to identify a point that is significantly distant from other points. Hence, the result of outlier detection is essentially privacy invasive; privacy protection and outlier detection have conflicting objectives. In order to reconcile the contradicting goals, we investigate two tasks, (1) counting outliers in a given subspace and (2) discovering subspaces containing many outliers, under the constraint of differential privacy.
 
 Subspace discovery for outlier analysis has been investigated as a major topic of outlier detection \cite{Knorr:1999:FIK:645925.671529,DBLP:conf/cikm/KellerMWB13,DBLP:conf/icde/KellerMB12}.
 The major motivation of existing subspace discovery methods was basically tackling the high dimensionality. Full-space outlier analysis might fail to detect outliers found only in specific sub-spaces because of a large number of irrelevant attributes~\cite{DBLP:conf/icde/KellerMB12,aggarwal2013outlier}.
 
 In this study, we solve the subspace discovery problem in order to balance privacy protection and outlier analysis. The found subspaces can be interpreted as knowledge to understand why the points are outliers and how such outliers are generated. After identifying the subspaces containing a large number of outliers, the number of the outliers are released. Our solutions presented in the following sections guarantees differential privacy for both tasks.

\subsection{Outlier Detection}
  \label{subsec:outlier}
  In this study, we use distance-based outliers~\cite{Knorr:1998:AMD:645924.671334}.
  Presuming that records are real-valued vectors, $\bm{x}_i \in \mathbb{R}^d$, and letting $X=\{\bm{x}_i\}_{i=1}^{N}$ denote the database, we let $S \in \{1,2, \hdots, d\}$ denote a subspace.
  The Euclidean distance between $\bm{x}, \bm{y} \in \mathbb{R}^d$ in subspace $S$ is denoted by $dist_S(\bm{x}, \bm{y})=\sqrt{\frac{\sum_{i \in S}(x_i-y_i)^2}{|S|}}$~\cite{DBLP:conf/icde/KellerMB12}. Then, the set of neighborhood vectors of $\bm{x}$ in subspace $S$ is defined as follows.
  
  \begin{definition}[Neighboring vector in subspace $S$]  \label{neighbor-vector}
   Let $r>0$ and $k \in \{1, \hdots, N\}$.
   Then, the set of neighboring vectors of $\bm{x} \in S$ is
   \begin{equation}
    \label{neighbor-vector-eq}
     N_S(X,r,\bm{x})=\{\bm{x} \in X | dist_S(\bm{x},\bm{y}) \leq r, \bm{x} \neq \bm{y}, \bm{y} \in X\}.
   \end{equation}
  \end{definition}
  With this definition of the neighboring vectors, the outliers are defined as follows.

  \begin{definition}[Outliers in subspace $S$]
   \label{outlier}
   Given threshold $k$ and radius $r$, the set of outliers of $X$ in subspace $S$ is
   \begin{equation}  \label{outlier-eq}
    O_S(X, k, r) = \{\bm{x} \in X | |N_S(X,r,\bm{x})| < k \}.
   \end{equation}
  \end{definition}
  
  Distance-based outliers are definable in this study with any type of object and distance defined for the corresponding objects, but we presume that the objects are represented as real vectors and that they use the Euclidean distance as the distance definition.

\subsection{Queries for Outlier Analysis}
As already discussed, we consider two tasks for outlier analysis, outlier count and subspace discovery.

Let $S \in 2^{\{1,2, \hdots, d\}}$ be a target subspace. Then, the task of outlier count is to find the number outliers in subspace $S$:
\begin{equation}
 \label{count}
  q_{count}(X, k, r, S) = |O_S(X, k, r)|.
\end{equation}
If the subspace is not specified, $O(X, k, r)$ denotes the set of outliers in the full dimension.

Let $\mathcal{S} \subseteq 2^{\{1,2, \hdots, d\}}$ be a subspaces.
The task of top-$h$ subspace discovery is to identify $h$ subspaces in $\mathcal{S}$ containing the $h$ largest number of outliers:
\begin{equation}
 \label{subspace}
  q_{subspace}(X, k, r, h, \mathcal{S}) = \{ S_{\pi(1)}, S_{\pi(2)}, \hdots, S_{\pi(h)}\}
\end{equation}
where $\pi: \{1,\hdots, |{\cal S}|\} \to \{1,\hdots, |{\cal S}|\}$ is a function that outputs the index of the subspace ordered by $q_{count}(X, k, r, S)$. For example, $\pi(i)$ denotes the index of the subspace containing the $i$th largest number of outliers.

  \subsection{Differential privacy of Outlier Analysis}
  \label{subsec:sinario}
  We introduce several typical scenarios of differentially private outlier analysis using the two types of queries, $q_{count}$ and $q_{subspace}$.
  
  {\bf Scenario 1.} Given threshold $k$ and radius $r$, suppose the objective is to inspect that the outliers exists in the given dataset. The analyst issues query $z = q_{count}(X,k,r)$, and then checking $z > 0$ yields the final result. Let $z' = q_{count}(X', k, r)$.
  For guarantee of $(\epsilon, \delta)$-differential privacy, we require, for $\forall X \sim X'$ and $\forall t \in \mathcal{T}$,
  \begin{align}
   Pr[t=\mathcal{A}(z)] \le e^\epsilon Pr[t=\mathcal{A}(z')] + \delta.
  \end{align}
  
  {\bf Scenario 2.} Let data dimension be $d=3$. Given threshold $k$ and radius $r$, suppose the objective is to identify the subspaces that cause the two largest number of outliers and learn the number of outliers in the two discovered subspaces. Then, the target subspace set is $\mathcal{S}=\{\{1\},\{2\},\{3\},  \{1,2\}, \{1,3\}, \{2,3\}, \{1, 2,3\}\}$.
  The analyst issues query $q_{subspace}(X, k, r, 2, \mathcal{S})$ and obtains the two subspaces $z_1=\{ S_{\pi(1)}, S_{\pi(2)} \}$. For each subspace, the analyst issues queries as $z_2 = q_{count}(X, k, r, S_{\pi(1)})$ and $z_3 = q_{count}(X, k, r, S_{\pi(2)})$.
  For guarantee of $(\epsilon, \delta)$-differential privacy, we require, for $\forall X \sim X'$ and $\forall t \in \mathcal{T}$,
  \begin{align}
   Pr[t=\mathcal{A}(z_1,z_2,z_3)] \le e^\epsilon Pr[t=\mathcal{A}(z_1', z'_2, z'_3)] + \delta,
  \end{align}
  where $z'_1,z'_2,$ and $z'_3$ are reposes learned from $X' \sim X$.

\section{Differentially Private Count of Outliers}
 \label{sec:DP_count}
 As explained in this section, we investigate the problem of differentially private count of outliers in a given subspace. The discussion herein holds for any subspace including the full space. Therefore, for this discussion, we presume that the outlier is counted in the full dimension.

\subsection{Difficulties in Global Sensitivity Method}
  \label{subsec:qcount-GS}
  Analytical evaluation of the global sensitivity of determination of $q_{count}$ is not trivial, partly because it needs the kissing number. The kissing number $K_d$ is the largest number of hyperspheres with same radius in $\mathbb{R}^d$ that can touch equivalent hyperspheres with no intersections~\cite{Musin_thekissing,Musin05thekissing,mittelmann2010high}.
  The kissing numbers in $d=1$ and $d=2$ are readily derived respectively as $K_1=2$ and $K_2=6$ (see Fig.~\ref{fig:upperbound} for $K_2=6$). However, finding the kissing number in $d \ge 3$ is not trivial. In addition, the kissing number in general dimensions remains as an open problem~\cite{Musin_thekissing,Musin05thekissing,mittelmann2010high}.
   \begin{figure}[t]
    \centering
    \includegraphics[width=7.5cm]{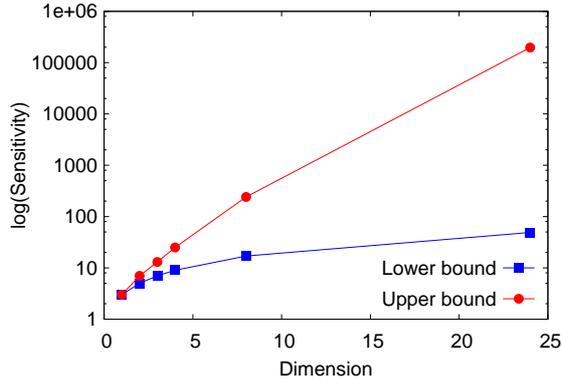}
    \caption{The bounds of the global sensitivity for counting outliers}
    \label{fig:GS-bound}
   \end{figure}
  We derive the upper and lower bound of the global sensitivity of $q_{count}$ presuming that the kissing number in general dimensions is given.
   \begin{theorem}[Upper and lower bound on the global sensitivity of $q_{count}$]
    \label{GS-Bound}
    Let $K_d$ be the kissing number in $\mathbb{R}^d$. Then, the upper and lower bound on the global sensitivity of $q_{count}$ is
    \begin{equation}
     \label{GS-eq}
      \min(N, 2dk + 1) \leq GS_{q_{count}, d}(k) \leq \min(N, kK_d + 1).
    \end{equation}
   \end{theorem}

\begin{proof}
The lower bound is trivial so we omit the proof. We show the proof for the upper bound.
In the problem of the kissing number, suppose the radius of the center hypersphere and the hyperspheres touching the center hyperspheres (referred to as the surrounding hyperspheres) are $r/2$.
The distance between the center point of the center hypersphere and those of the surrounding hyperspheres are $r$.
Noting that no intersection between the surrounding hyperspheres does not exist, the distance between the center point of any two surrounding hyperspheres are equal to or greater than $r$ (the equality holds if the two surrounding hyperspheres are touching).

Suppose $x_0$ be the center of the center hypersphere and $x_1$ be the center of a surrounding hypersphere that does not touch any other surrounding hyperspheres. We further suppose $k-1$ datapoints exist at exactly the same location as $x_1$, that is, $x_1 = x_2 = \hdots = x_k$.
Letting $X=\{ x_0, x_1, \hdots, x_k \}$, $q_{count}(X,k,r)=0$ because all the $k+1$ points are within a hypersphere of radius $r$. If $x_0$ is removed from $X$ as $X'=\{ x_1, \hdots, x_k \}$, $q_{count}(X,k,r)=k+1$ holds because the remaining $k$ points do not have $k$ neighbor vectors and $x_0$ itself can be an outlier after moved.

By definition of the kissing number, the number of the surrounding hyperspheres that does not touch mutually is at most $\mathcal{K}_d$. By applying the setting described above for each surrounding hypersphere, we have a database that holds $q_{count}(X,k,r)=0$, but after moving of $x_0$, $q_{count}(X,k,r)=k \mathcal{K}_d +1$.
Noting that no more hyperspheres cannot be packed around $x_0$, this is the upper bound of the outlier count.
\end{proof}

   \begin{figure}[t]
    \centering
    \includegraphics[width=6cm]{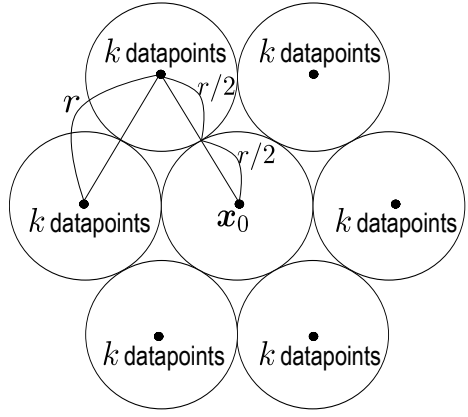}
    \caption{This figure shows an example of the upper bound of the global sensitivity in two dimension. Six surrounding hyperspheres can be packed around the center hypersphere because the kissing number is $K_2=6$. We here suppose $k$ datapoints exist at the center of each surrounding hypersphere and no datapoint exists at $\bm{x}_0$, the center of the center hypersphere. Then, $kK_2$ outliers become inliers by adding a point to $\bm{x}_0$. Suppose the added point is an outlier, Then, the added point can be changed from an outlier to an inlier, too. The upper bound of the global sensitivity for two dimension is thus $kK_2 + 1 = 6k + 1$.}
    \label{fig:upperbound}
   \end{figure}

   We empirically investigate the tightness of the bound in low dimensions.
   In $d=1$ and $d=2$, the global sensitivity is given respectively as $GS_{q_{count},1}(k) = 2k + 1$ and $  GS_{q_{count},2}(k) = 5k + 1$. Noting that $K_1=2$ and $K_2=6$, the bound is tight in $d=1$  but not in $d=2$.
   Fig.~\ref{fig:GS-bound} shows the upper and lower bounds of the global sensitivity of $q_{count}$ evaluated using known upper bounds on the kissing number~\cite{Musin_thekissing,Musin05thekissing,mittelmann2010high}.
   As the figure shows, the upper bound of the global sensitivity grows exponentially with respect to the dimensionality, which indicates that the guarantee of differential privacy by perturbation based on the global sensitivity can be impractical, especially when the dimensionality of the target subspace is large.

   The global sensitivity can be prohibitively large simply because the global sensitivity is evaluated considering the worst case. However, one can typically expect that the number of outliers in the database is much smaller than the number of instances. To improve the utility of the count query, we introduce the smooth sensitivity, which is a sensitivity definition depending on the database.

  \subsection{Local Sensitivity and Smooth Sensitivity}
  \label{subsec:localAndSmooth}
  For convenience of discussion later, several notations are introduced here.
  Given radius $r$, $deg(\bm{x})$ denotes the size of neighborhoods of $\bm{x}$:
  \begin{equation} \label{degree-eq}
   deg(X,r,\bm{x}) = |N(X,r,\bm{x})|.
  \end{equation}
  We say that the degree of $\bm{x}$ is $k$ if $deg(X,r,\bm{x})=k$.
  A set of vectors in $X$ whose degree is exactly $k$ is denoted as
  \begin{equation} \label{degree-set-eq}
   V(X,k,r) = \{\bm{x} \in X : deg(\bm{x})=k\}.
  \end{equation}
  Unless specifically stated otherwise, the radius $r$ and target database $X$ is fixed. Therefore, they are omitted as $deg(\bm{x})$ and $V(k)$.
  Finally, a set of degree-$k$ neighborhoods of $\bm{x}$ in $X$ is denoted as
  \begin{equation}    \label{cv-eq}
   CV(X, \bm{x}, k, r) = B(\bm{x}, r) \cap V(k),
  \end{equation}
  where $B(\bm{x}, r)$ denotes the sphere with radius $r$ and centered at $\bm{x}$.
  
\subsubsection{Local Sensitivity}
\label{subsubsec:qcount-LS}
Given database $X$, let $X_1$ be a database s.t. $H(X, X_1)=1$.
Then, following the definition of the local sensitivity in Section \ref{subsubsec:smooth}, the local sensitivity of $q_{count}$ is defined as
  \begin{align}
   LS_{q_{count}}^{(0)}(X, k, r) = \max_{X_1:H(X,X_1)=1} \|q_{count}(X_0,k,r) - q_{count}(X_1,k,r) \|.
  \end{align}
  Exact evaluation of the exact local sensitivity is intractable.
  Instead, the following theorem gives the upper bound of the local sensitivity.
 \begin{theorem}\label{LS}
  Given $X$, the local sensitivity of $q_{count}$ for $X$ is bounded above as
    \begin{align}
     \label{LS-eq}
     LS_{q_{count}}^{(0)}(X,  k, r) \leq \max\left\{ \max_{\bm{x} \in X}\{ |CV(X, \bm{x}, k, r)|  \},  \max_{\bm{x} \in \mathcal{D}}\{ |CV(X, \bm{x}, k-1, r)| \} \right\} + 1.
    \end{align}
 \end{theorem}
\begin{proof}
Intuitively, $CV(X_0, \bm{x}, k, r)$ represents the set of non-outliers that become outliers if $\bm{x}$ is removed; $CV(X_0, \bm{x}, k-1, r)$ is the set of outliers that become inliers if a vector is placed at $\bm{x}$. Thus, if vector $\bm{x}_0 \in X_0$ is moved to $\bm{x}'_0$, the number of outliers increases by $|CV(X_0, \bm{x}_0, k, r)|$ by removing $\bm{x}_0$ and the number of inliers decreases by $|CV(X_0, \bm{x}'_0, k-1, r)|$ by adding $\bm{x}'_0$. With this understanding, the local sensitivity is given as:
\begin{align}
 & LS_{q_{count}}^{(0)}(X_0, k, r) \\
=& \max_{X_1:H(X_0,X_1)=1} \|q_{count}(X_0,k,r) - q_{count}(X_1,k,r) \| \\
 \leq& \max_{\bm{x}_0 \in X_0, \bm{x}'_0 \in S} |CV(X_0, \bm{x}_0, k, r) \setminus CV(X_0, \bm{x}'_0, k-1, r)|  + 1  \\
 \leq&  \max_{\bm{x}_0 \in X_0, \bm{x}'_0 \in S} \max\left\{|CV(X_0, \bm{x}_0, k, r)|, |CV(X_0, \bm{x}'_0, k-1, r)| \right\} + 1  \\
 =& \max\left\{ \max_{\bm{x} \in X_0}\{CV(X_0, \bm{x}, k, r)  \},  \max_{\bm{x}'_0 \in S}\{ CV(X_0, \bm{x}, k-1, r)\} \right\} + 1.
\end{align}
\end{proof}

\noindent Naive evaluation of the local sensitivity is intractable. An algorithm to evaluate this upper bound is presented in Section \ref{subsec:efficientComputation}.

\subsubsection{Smooth Sensitivity}
\label{subsubsec:qcount-SS}
Given database $X$, let $X_t$ be a database s.t. $H(X_0, X_t)=t$.
By definition, the smooth sensitivity of $q_{count}$ is given as
 \begin{align}
  S^*_{q_{count}}(X) = \max_{t=0,1,\hdots, n} e^{-t\epsilon} LS^{(t)}_{q_{count}}(X) ,
 \end{align}
 where
 \begin{align}
  LS^{(t)}_{q_{count}}(X) = \max_{X_t: H(X,X_t) = t} LS^{(0)}_{q_{count}}(X_t).
 \end{align}
 Here, $X_t$ denotes a database s.t. $H(X,X_t)=t$. The function $LS^{(t)}_q(X)$ returns the largest local sensitivity among the datasets of which $t$ records differ from $X$. Similarly to $LS^{(0)}_{q_{count}}(X)$, exact evaluation of $LS^{(t)}_{q_{count}}(X)$ is intractable because the variation of $X_t$ can increase exponentially with respect to $t$. Instead, we derive the upper bound on $LS^{(t)}_{q_{count}}(X)$ using $CV(X,\bm{x}, k,r)$.
 \begin{theorem}
  \label{thm:ls-a-bound}
  Given $X$, for $t \ge 0$, $LS^{(t)}_{q_{count}}(X)$ is bounded above as
  \begin{align}
   LS^{(t)}_{q_{count}}(X) \le \max_{\bm{x} \in \mathcal{D}}\left\{ \max\{C^{(t)}(X,\bm{x},k,r), C^{(t)}(X,\bm{x},k-1,r)\} + t + 1 \right\}, \label{eq:ls-t-upper-bound}
  \end{align}
  where
  \begin{align}
   C^{(t)}(X,\bm{x},k,r) = \left| \bigcup_{i = -t}^t CV(X,\bm{x},k+i,r) \right|.
  \end{align}
 \end{theorem}
 For the proof of this theorem, we use the following helper lemma.

 \begin{lemma}
  \label{lem:a-neibhor-sub}
  Let $t \ge 0$ be an integer, and let $X$ and $X_t$ be databases such that $H(X,X_t) = t$. Then, for any $\bm{x} \in \mathcal{D}$, threshold $k$, and radius $r$,
  \begin{align}
   |CV(X_t,\bm{x},k,r)| \le \left| \bigcup_{i = -t}^t CV(X,\bm{x},k+i,r) \right| + t.
  \end{align}
 \end{lemma}

\begin{proof}[Proof of Lemma \ref{lem:a-neibhor-sub}]
We first consider the case $t = 1$. Suppose $\bm{x} \in X$ is moved from $\bm{x}$ to $\bm{x}_1$, and $X_1$ is given as $X_1 = X\setminus\{\bm{x}\}\cup\{\bm{x}_1\}$. The degree of records in $X\setminus\{\bm{x}\}$ around $\bm{x}$ is decreased by one by removing $\bm{x}$, and the degree of records in $X\setminus\{\bm{x}\}$ around $\bm{x}_1$ is increased by one by adding $\bm{x}_1$. Since the degree of the records in $V(X,k+1,r)$ and $V(X,k-1,r)$ may become $k$ in $X_1$, $V(X_1,k,r)$ is thus a subset of $V(X,k+1,r) \cup V(X,k,r) \cup V(X,k-1,r) \cup \{x_1\}$. When $t > 1$, for the same reason, $V(X_t,k,r)$ is a subset of $\bigcup_{i=-t}^t V(X,k+i,r) \cup \{\bm{x}_1, \bm{x}_2, ..., \bm{x}_t \}$ where $\bm{x}_1, ..., \bm{x}_t$ are the records moved from $X$ to $X_t$. Thus, the size of $CV(X_t, \bm{x},r,k)$ is bounded above as
\begin{align}
  |CV(X_t, \bm{x},r,k)| \le& \left| B(\bm{x}, r) \cap \bigcup_{i=-t}^t V(X,k+i,r) \cup \{\bm{x}_1, \bm{x}_2, ..., \bm{x}_t \} \right| \\
  \le& \left| \bigcup_{i=-t}^t B(\bm{x}, r) \cap V(X,k+i,r) \right| + |\{\bm{x}_1, \bm{x}_2, ..., \bm{x}_t \} | \\
  \le& \left| \bigcup_{i=-t}^t CV(X,\bm{x},k+i,r) \right| + t.
\end{align}
\end{proof}

By using Lemma \ref{lem:a-neibhor-sub}, we now prove the Theorem \ref{thm:ls-a-bound}.
\begin{proof}[Proof of Theorem \ref{thm:ls-a-bound}]
As proved in Theorem \ref{LS}, the local sensitivity of the query $q_{count}$ is bounded above by
\begin{align}
  LS^{(0)}_{q_{count}}(X) \le \max\{\max_{\bm{x} \in X}|CV(X,\bm{x},r,k)|, \max_{\bm{x} \in \mathcal{D}}| CV(X,\bm{x},k-1,r)| \} + 1.
\end{align}
From exchangeability of $\max$, letting
 \begin{align}
  C^{(t)}_{\rm out}(X,k,r) =& \max_{X_t:H(X,X_t) = t}\max_{\bm{x} \in X_t}|CV(X_t,\bm{x},r,k)|, \mbox{ and } \\
  C^{(t)}_{\rm in}(X,k-1,r) =& \max_{X_t:H(X,X_t) = t}\max_{\bm{x} \in \mathcal{D}}|CV(X_t,\bm{x},r,k-1)|,
 \end{align}
 yields
 \begin{align}
  LS^{(t)}_{q_{count}}(X) \le \max\{ C^{(t)}_{\rm out}(X,k,r), C^{(t)}_{\rm in}(X,k-1,r) \} + 1. \label{eq:ls-a-c-bound}
 \end{align}
 $C^{(t)}_{\rm out}(X,k,r)$ can be bounded above using $C^{(t)}_{\rm in}(X,k-1,r)$ as
 \begin{align}
  C^{(t)}_{\rm out}(X,k,r) =& \max_{X_t:H(X,X_t) = t} \max_{\bm{x} \in X_t}|CV(X_t,\bm{x},r,k)| \\
  \le& \max_{X_t:H(X,X_t) = t} \max_{\bm{x} \in \mathcal{D}}|CV(X_t,\bm{x},r,k)| \\
  =& C^{(t)}_{\rm in}(X,k,r). \label{eq:cout-cin}
 \end{align}
Here, we use the fact that $X' \subset \mathcal{D}$. By Lemma \ref{lem:a-neibhor-sub}, we have
 \begin{align}
  C^{(t)}_{\rm in}(X,k,r) \le& \max_{\bm{x} \in \mathcal{D}} \left| \bigcup_{i = -t}^t CV(X,\bm{x},k+i,r) \right| + t . \label{eq:cin-bound}
 \end{align}
By substituting eqs. \ref{eq:cout-cin} and \ref{eq:cin-bound} into eq. (\ref{eq:ls-a-c-bound}), we get the claim.
\end{proof}

\subsection{Efficient Computation of Smooth Sensitivity Bound}
\label{subsec:efficientComputation}

For randomization by the mechanism of Theorem \ref{LM}, it is necessary to evaluate the smooth upper bound. Naive evaluation of the smooth upper bound of eq. (\ref{eq:ls-t-upper-bound}) is intractable because it requires an exhaustive search over continuous domain to evaluate $LS^{(t)}_{q_{count}}(X)$. To alleviate this, we first show an efficient algorithm that evaluates the upper bound of $LS^{(t)}_{q_{count}}(X)$ shown derived by Theorem \ref{thm:ls-a-bound}. Then using the algorithm, we derive the algorithm that calculates the smooth sensitivity upper bound.

\subsubsection{Algorithm for local sensitivity bound} 
To evaluate the upper bound of $LS^{(t)}_{q_{count}}(X)$, we need to calculate
  \begin{align}
   \max_{\bm{x} \in \mathcal{D}} C^{(t)}(X,\bm{x},k,r) =& \max_{\bm{x} \in \mathcal{D}} \left| \bigcup_{i = -t}^t V(X,k+i,r) \cap B(\bm{x},r) \right|, \mbox{ and} \label{eq:ct-k} \\
   \max_{\bm{x} \in \mathcal{D}} C^{(t)}(X,\bm{x},k-1,r) =& \max_{\bm{x} \in \mathcal{D}} \left| \bigcup_{i = -t}^t V(X,k+i-1,r) \cap B(\bm{x},r) \right|. \label{eq:ct-k-1}
  \end{align}

  Letting $P = \bigcup_{i = -t}^t V(X,k+i,r)$ (resp. $P = \bigcup_{i = -t}^t V(X,k+i-1,r)$), we can obtain the value of eq.(\ref{eq:ct-k}) (resp. eq.(\ref{eq:ct-k-1})) by finding the largest subset $C \subseteq P$ that is enclosed by a ball with radius $r$. To check whether or not a given subset $C \subseteq P$ is enclosed by the ball, we use the algorithm that solves the {\em smallest enclosing ball}~(seb) problem~\cite{Fischer03fastsmallest-enclosing-ball}. The goal of the problem is to find the smallest ball that encloses the given points.
  The given subset $C \subseteq P$ is enclosed by a ball with radius $r$ if ${\rm seb}(C) \le r$ where ${\rm seb}(C)$ denotes the radius of the resultant ball of the smallest enclosing ball problem of $C$.

  Algorithm \ref{enumerate-algo} shows the recursive algorithm that calculates  eq. (\ref{eq:ct-k}) or eq. (\ref{eq:ct-k-1}) for given $P = \bigcup_{i = -t}^t V(X,k+i,r)$ or $P = \bigcup_{i = -t}^t V(X,k+i-1,r)$. $P[i]$ denotes the $i$-th element of the set $P$. Algorithm \ref{enumerate-algo} searches for the largest subsets $C \subseteq P$ that is enclosed by a ball with radius $r$ with the breadth-first search. In the algorithm, the calls of ${\rm seb}$ can be skipped for efficiency by using the fact that the radius of the enclosing ball of $C_2$ is larger than one of $C_1$ if $C_1 \subseteq C_2 \subseteq P$.
The computational cost of Algorithm~\ref{enumerate-algo} is $\mathcal{O}(2^{|P|})$ of the calls of ${\rm seb}$.

\begin{algorithm}[t]
 \caption{Calculation of $\max_{\bm{x} \in \mathcal{D}} C^{(t)}(X,\bm{x},k,r)$(eq.~{\protect(\ref{eq:ct-k})} and eq.~{\protect(\ref{eq:ct-k-1}))}}\label{enumerate-algo}
 \SetKwProg{Fn}{Function}{}{end}
 \SetKwFunction{FE}{E}%
 \SetKwInOut{KwInit}{Initialization}
 \DontPrintSemicolon
 \KwIn{Records $P$ and radius $r$.}
 \KwOut{The value of eq.(\ref{eq:ct-k}) or eq.(\ref{eq:ct-k-1}).}
 \KwInit{$C = \emptyset$ and $i = 1$}
 \Fn{\FE{$r, P, C, i$}}{
 $br \gets 0$ \;
 \If{$C \neq \emptyset$}{
 $br \gets {\rm seb}(C)$ \;
 }
 \If{$br \leq r$}{
 $m \gets |C|$ \;
 \If{$i \leq |P|$}{
 $b_1 \gets$ \FE{$r, P, C \cup \{P[i]\}, i+1$} \;
 $b_2 \gets$ \FE{$r, P, C, i+1$} \;
 $m \gets \max\{m, b_1, b_2\}$ \;
 }
 \Return{$m$}
 }\Else{
 \Return{$0$}
 }
 }
\end{algorithm}

\subsubsection{Algorithm for smooth sensiticity bound.}
Algorithm~\ref{enumerate-algo} costs exponential time with respect to $|P|$ and the size of $P$ increases monotonically as $t$ increases.
However, because of exponential decrease of $e^{-t\beta}$, maximization of $e^{-t\beta}LS^{(t)}_{q_{count}}(X)$ is attained by small $t$ in most cases.
Taking account of this property, we provide Algorithm \ref{alg:efficient} that calculates the smooth sensitivity bound with avoiding evaluation of $LS^{(t)}_{q_{count}}(X)$ of large $t$.
  \begin{proposition}
    \label{prop:lst-bound}
    For any $t$ and $t' < t$, $LS^{(t)}_{q_{count}}$ is bounded above as
    \begin{align}
      LS^{(t)}_{q_{count}}(X) \le \min\{N, \max\{U^{(t)}_{t'}(X,k,r), U^{(t)}_{t'}(X,k-1,r)\} + t + 1\},
    \end{align}
    where
    \begin{align}
      U^{(t)}_{t'}(X,k,r) = \max_{\bm{x} \in \mathcal{D}} C^{(t')}(X,\bm{x},k,r) + \left|\bigcup_{i \in \{-t,...,-t'-1\} \cup \{t'+1,...,t\}} V(X,k+i,r) \right|.
    \end{align}
  \end{proposition}
  \begin{proof}[Sketch of proof.]
  For any database $X$, because the number of outliers does not exceed the number of the records in $X$, the local sensitivity is less than $N$. In addition, using the fact that $CV(X,\bm{x},k,r) \subseteq V(X,k,r)$ for any $\bm{x} \in \mathcal{D}$, we can derive $\max_{\bm{x} \in \mathcal{D}} C^{(t)}(X,\bm{x},k,r) \le U^{(t)}_{t'}(X,k,r)$ for any $t$ and $t' < t$.
  \end{proof}
  Using the bound in Proposition \ref{prop:lst-bound}, we have the upper bound of $e^{-t\beta}LS^{(t)}_{q_{count}}(X)$ as
  \begin{align}
   e^{-t\beta}LS^{(t)}_{q_{count}}(X) \le& e^{-t\beta}\min\{N, \max\{U^{(t)}_{t'}(X,k,r), U^{(t)}_{t'}(X,k-1,r)\} + t + 1\} \\
    :=& S^{t',t}_{\rm UB}(X).
  \end{align}
  Letting $S^t_{\rm UB}(X) = \max_{i = 1,...,N-t} S^{t,t+i}_{\rm UB}(X)$, we can obtain the following proposition.
 \begin{proposition}\label{proposition:upper-prune}
  If there exists $U_T$ such that $\max_{t = 0,...,T} e^{-t\beta}LS^{(t)}_{q_{count}}(X) \le U_T$ and $S^{T}_{\rm UB}(X) \le U_T$, then $S^*_{q_{count}}(X) \le U_T$.
 \end{proposition}
\begin{proof}
If $S^T_{\rm UB}(X) = \max_{i = 1,...,N-T} S^{T, T+i}_{\rm UB}(X) \le U_T$, since $e^{-t\beta}LS^{(t)}_{q_{count}}(X) \le S^{T,t}_{\rm UB}(X)$ for any $t > T$, we have $e^{-t\beta}LS^{(t)}_{q_{count}}(X) \le U_T ~, \forall t > T$. Thus, we have $\max_{t = 0,...,T} e^{-t\beta}LS^{(t)}_{q_{count}}(X) \le U_T$ and $\max_{t > T} e^{-t\beta}LS^{(t)}_{q_{count}}(X) \le U_T$.
\end{proof}

\noindent Proposition~\ref{proposition:upper-prune} shows that if the largest upper bound in Theorem \ref{thm:ls-a-bound} for $t = 0, ..., T$ can be bounded above by $S^T_{\rm UB}(X) $, then the calculation of the upper bound in Theorem \ref{thm:ls-a-bound} for $t > T$ can be skipped. Algorithm \ref{alg:efficient} shows the calculation of the smooth sensitivity of $q_{count}$ with this skip by following Proposition~\ref{proposition:upper-prune}.

\begin{algorithm}[t]
 \caption{Calculation of the smooth sensitivity of $q_{count}$}
 \label{alg:efficient}
 \SetKwInOut{KwInit}{Initialization}
 \DontPrintSemicolon
 \KwIn{Database $X$, threshold $k$, radius $r$ and smooth parameter $\epsilon$.}
 \KwOut{The smooth sensitivity upper bound of query $q_{count}$ for database $X$.}
 \KwInit{$S_{\max} = 0$ and $\max_{\bm{x}\in\mathcal{D}} C^{(-1)}(X,\bm{x},k,r) = \max_{\bm{x}\in\mathcal{D}} C^{(-1)}(X,\bm{x},k-1,r) = 0$.}
 \For{$t = 0$ to $N$}{
  Calculate $S^{t-1}_{\rm UB}$ by Proposition~\ref{proposition:upper-prune} \;
  \If{$S^{t-1}_{\rm UB} \le S_{\max}$}{
   \Return{$S_{\max}$}
  }
  $S_{\max} \gets \max\{S_{\max}, e^{-t\beta}LS^{(t)}_{q_{count}}(X) \}$ \;
  Store $\max_{\bm{x}\in\mathcal{D}} C^{(t)}(X,\bm{x},k,r)$ and $\max_{\bm{x}\in\mathcal{D}} C^{(t)}(X,\bm{x},k-1,r)$ for calculating $S^{t}_{\rm UB}$ in next loop
 }
 \Return{$S_{\max}$}
\end{algorithm}

\section{Differentially Private Discovery and Detection}
 \label{sec:discovery_detection}
 We are able to get the number of outliers in the database while ensuring $(\epsilon, \delta)$-differential mechanism by previous technique. Next, we try to achieve analyzing like Scenario 2 and 3. We descrive how to achieve Scenario 2 in Section~\ref{subsec:discovery} and Scenario 3 in Section~\ref{subsec:detection}.

  \subsection{Top-$h$ Subspace Discovery with Exponential Mechanism}
  \label{subsec:discovery}
 This section investigates differential privacy for finding subspaces that contains outliers.
 As already discussed, identification of outliers and protecting privacy of instances are intrinsically incompatible. However, if the interest of the analyst is simply to learn the situations that the outliers appear, we can alleviate releasing the outlier counts on each subspace; releasing subspaces containing many outliers would suffice~\cite{Knorr:1999:FIK:645925.671529}. In this section, we present another mechanism for differential privacy that allows us to learn top-$h$ subspaces that contains outliers. 

  The subspace containing many outliers can be simply found by issuing count queries for each subspace. However, responses obtained with such a procedure can be useless in typical settings.
  Let ${\cal F}_c$ be the set of subspaces spanned by $c$ dimensions. Then, the privacy parameters need to be set as $(\epsilon/|S|, \delta/|S|)$ for each count query so that the entire process achieves $(\epsilon, \delta)$-differentially private because of sequential composition.
  The privacy budget can be saved by applying the more sophisticated composition theorem~\cite{dwork2010boosting}. However, noting that the size of ${\cal F}_c$ can grow exponentially in $d$, it is still difficult to manage high dimensionality.
  
  We consider the problem of the top-$h$ subspace discovery by means of the exponential mechanism, which allows us to avoid releasing outlier counts for each subspace. 
  Let $d$ be the instance dimension and let $E=\{1,2,\cdots, d\}$.
  Then, the set of $c$-dimensional subspaces is denoted as ${\cal F}_c = \{ F_c| F_c \subset E, |F_c|=c  \}$.
  Given $c$ and $E$, the top-$h$ subspace discovery is the problem to find the $h$ subspaces in ${\cal F}_c$ containing the $h$ largest number of outliers.
  The exponential mechanism can be used to release discrete values with achieving differential privacy. We employ the following function as the utility function for the top-$h$ subspace discovery:
   \begin{equation}
    \label{eq:uf-subspace}
     u_{subspace}(X,r,k,S) = \frac{q_{count}(X,k,r,S)}{GS_{q_{count},|S|}^{\rm{UB}}(k)}
   \end{equation}
   where $GS_{q_{count},|S|}^{\rm{UB}}(k)$ denotes the upper bound of the global sensitivity of the count query, as derived by eq. (\ref{GS-eq}) in Section \ref{subsec:qcount-GS}. The following theorem denotes the differential privacy achieved by the exponential mechanism with this utility function.
   \begin{theorem}
    \label{theo:uf-subspace}
    The exponential mechanism with utility function $u_{subspace}$ achieves $2\epsilon$-differential privacy.
   \end{theorem}
   
   \begin{proof}
    The global sensitivity of utility function $u_{subspace}$ is given as
    \begin{equation}
     \label{ap:uf-subspace}
      \begin{split}
       \Delta u_{subspace} &= \max_{S \in {\cal F}_c} \max_{X,X'\in \mathcal{D}^{n}:H(X,X')=1} ||u_{subspace}(X,k,r,S)-u_{subspace}(X',k,r,S)|| \\
       &= \max_{S \in {\cal F}_c} \max_{X,X'\in \mathcal{D}^{n}:H(X,X')=1} \frac{||q_{count}(X,k,r,S) -q_{count}(X',k,r,S)|| }{GS_{q_{count},|S|}^{\rm{UB}}(k)} \\
       &= \frac{GS_{q_{count},|S|}(k)}{GS_{q_{count},|S|}^{\rm{UB}}(k)} \leq 1.
      \end{split}
    \end{equation}
    Hence, we have $\Delta u_{subspace} \leq 1$. The exponential mechanism with utility function $u_{subspace}$ thus achieves $2\epsilon \Delta u_{subspace}$-differential privacy, which concludes the proof.
   \end{proof}
   
   To obtain the top-$h$ (suspected) subspaces, we need to iterate the exponential mechanism until $h$ different subspaces are found. Algorithm \ref{alg:exp} denotes the entire procedure for the top-$h$ query with $\epsilon$-differential privacy. Therefore, top-$h$ discovery of subspace query could also be adupted. At line 2-4, the utility for each subspaces in ${\cal F}_c$ are evaluated. At line 5-10, $h$ subspaces are chosen by iterative application of exponential mechanism $\epsilon^{\epsilon / h}_{u_{subspace}}$. Finally, the selected subspaces are released. Note that the privacy parameter for the exponential mechanism is set to $\epsilon/h$ so that the entire procedure of the top-$h$ subspace discovery achieves $\epsilon$-differential privacy. 

\begin{algorithm}[t]
 \caption{Mechanism of Top-$h$ Query}    \label{alg:exp}
 \SetKwProg{Fn}{Function}{}{end}
 \SetKwInOut{KwInit}{Initialization}
 \DontPrintSemicolon
 \KwIn{Top-$h$ query $q_h$ and smooth parameter $\epsilon$.}
 \KwOut{The top-$h$ items $R_h$.}
 \KwInit{$R_h \leftarrow \emptyset$}
 \Fn{$q_{h}(R,h,\cdot)$}{
 \For{$r \in R$}{
 calculate the utility of a item $r$ by $u_{h}(r, \cdot)$ \;
 }
 \While{$|R_h|<h$}{
 \Repeat{$r \notin R_h$}{
 $r \leftarrow \epsilon^{\epsilon/h}_{u_{h}, \Delta u_{h}}$ \;
 }
 $R_h \leftarrow \{r\} \cup R_h$ \;
 }
 \Return{$R_h$}
 }
\end{algorithm}
       
 \section{Experiments}
 \label{sec:experiment}
 In this section, we show the empirical evaluation of the utility of the mechanism for counting outliers query, discovery of subspace query and detection of outlier query.

\subsection{Settings}
\label{subsec:settings}

We conducted the experiments on some synthetic and real datasets for Scenario 1 and Scenario 2. As real datasets, we used Adult and Ionosphere datasets chosen from UCI Machine Learning Repository~\cite{Lichman:2013} which are originally prepared for classification tasks. 
For adapting outlier analysis, we carried out the preprocessing to the datasets in the same manner of \cite{DBLP:journals/corr/abs-0903-3257,Pham:2012:NTA:2339530.2339669}.
These datasets were scaled so that the average and variance of each attribute is $0$ and $1$, respectively. For Adult, we removed two categorical attributes, ``category'' and ``fnlwgt''.

The experiments for Scenario 1 were carried out on the two datasets, named Synthetic 1 and Adult 1.
Synthetic 1 consists of $50$ samples of $2$ dimensional real vectors, which contains $45$ inliers and $5$ outliers. The inliers are sampled from $\mathcal{N}(\bm{0},\mathbb{I})$ where $\mathbb{I}$ represents an identity matrix, and the outliers are sampled from $\mathcal{N}(\bm{\mu},\Sigma)$ where $\mu_1 = \mu_2 = 20$, and $\Sigma$ is a diagonal matrix such that $\Sigma_{11} = \Sigma_{22} = 100$.
Adult 1 is a subset of the original Adult dataset which contains $45$ positive labeled samples and $5$ negative labeled samples. The positive labeled samples and the negative labeled samples are treated as inliers and outliers, respectively~(See Table \ref{dataset-parameta1} for the detail).

The experiments for Scenario 2 were conducted on the three datasets, named Synthetic 2, Adult 2 and Ionosphere.
Synthetic 2 consists of $500$ samples of $10$ dimensional real vectors. The dataset contains $490$ inliers sampled from $\mathcal{N}(\bm{0},\mathbb{I})$ and $10$ outliers sampled from $\mathcal{N}(\bm{\mu},\Sigma)$ where $\mu_1 = \mu_2 = 20$, $\mu_i = 0$ for $i = 3...,10$, $\mathbb{I}$ represents an identity matrix, and $\Sigma$ is a diagonal matrix such that $\Sigma_{11} = \Sigma_{22} = 100$ and $\Sigma_{ii} = 1$ for $i = 3...,10$.
Adult 2 and Ionosphere are subsets of the original Adult or Ionosphere datasets which contains $490$ positive labeled samples and $10$ negative labeled samples in Adult 2, and $225$ positive labeled samples and $10$ negative labeled samples in Ionosphere. The treatment of inliers and outliers in real datasets is same as the experiments for Scenario 1~(See Table \ref{dataset-parameta2} for the detail).

\begin{table}[t]
 \centering
 \caption{Sumarry of datasets and parametas for Scenario 1}
 \begin{tabular}{c||c|c} \hline
  & Synthetic 1 & Adult 1 \\ \hline
  The number of outliers & $5$ & $5$\\
  The number of inliers & $45$ & $45$\\
  The number of samples $N$ & $50$ & $50$\\
  Dimension $d$ & $2$ & $7$\\
  Treshold $k$ & $3$ & $3$\\
  Radious $r$ & $1.1$ & $0.35$\\ \hline
 \end{tabular}
\label{dataset-parameta1}
\end{table}
\begin{table}[t]
  \centering
 \caption{Sumarry of datasets and parametas for Scenario 2 and Scenario 3}
 \begin{tabular}{c||c|c|c} \hline
  & Synthetic 2 & Adult 2 & Ionosphere\\ \hline
  The number of outliers & $10$ & $10$ & $10$\\ 
  The number of inliers & $490$ & $369$ & $225$\\
  The number of samples $N$ & $500$ & $379$ & $235$\\
  Dimension $d$ & $10$ & $7$ & $34$\\
  Treshold $k$ & $3$ & $3$ & $3$\\
  Radious $r$ & $0.13$ & $0.02$ & $0.06$\\ \hline
 \end{tabular}
 \label{dataset-parameta2}
\end{table}

\subsection{Count Outliers}
  \label{subsec:experiment-scenario1}
  Following the Scenario 1 described in Section~\ref{subsec:sinario}, we evaluated the utility of the mechanisms of $q_{count}$ on the synthetic dataset.
We changed the privacy parameter from $\epsilon=0.1$ to $0.9$; $\delta$ was fixed as $\delta=0.01$. See Table \ref{dataset-parameta1} for the parameters of the outliers.
We partitioned the instances into two classes: one is ``true'', indicating the instance detected as an outlier; the other is ``false''.
For each dataset, we tuned the radius $r$ so that the $Accuracy$ given by eq.(\ref{accuracy}) is maximized:
  \begin{equation}
   \label{accuracy}
    Accuracy = \frac{TP + TN}{TP + FP + FN + TN},
  \end{equation}
where $TP$, $TN$, $FP$ and $FN$ respectively denote true positive, true negative, false positive, and false negative. For implementation, we used \cite{miniball} to solve the smallest enclosing ball problem.
As the criterion of the utility of the mechanisms, we show the standard deviation of the noise added to the query. We compared the standard deviation of the noise of the mechanism based on the smooth sensitivity upper bound in eq.(\ref{eq:ls-t-upper-bound}) with the mechanism based on the global sensitivity lower bound in eq.(\ref{GS-eq}). Fig.~\ref{synthetic1} and Fig.~\ref{adult1} show true the number of outliers in the database and the standard deviations ($\sigma_{Global}$ and $\sigma_{Smooth}$) of the gaussian for each $\epsilon$. In Fig.~\ref{synthetic1} and Fig.~\ref{adult1}, ``Global'' and ``Smooth'' respectively present the global sensitivity-based mechanism and the smooth sensitivity-based mechanism.
\begin{figure}[t]
 \centering
 \includegraphics[width=150mm]{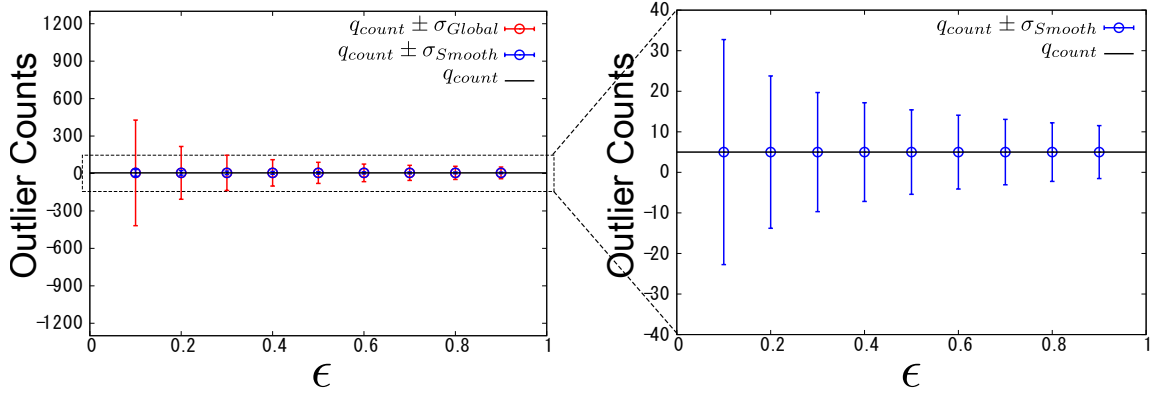}
 \caption{The result of Synthetic 1 on Scenario 1}
 \label{synthetic1}
\end{figure}

\begin{figure}[t]
 \centering
 \includegraphics[width=150mm]{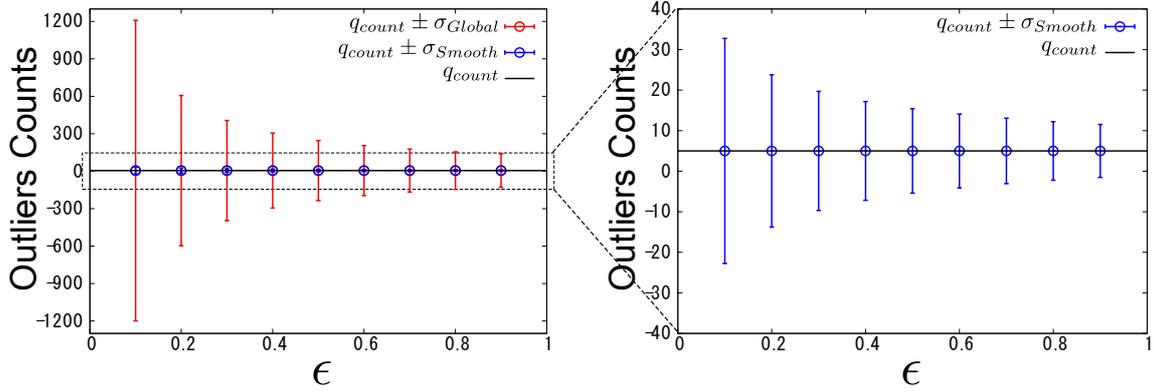}
 \caption{The result of Adult 1 on Scenario 1}
 \label{adult1}
\end{figure}
  It is apparent that the standard deviation of the noise of the smooth sensitivity-based mechanism is significantly lower than that of the global sensitivity-based mechanism. Indeed, the standard deviation of the noise of global sensitivity-based mechanism is approximately 10-30 times larger than that of the smooth sensitivity-based mechanism even though the global sensitivity-based mechanism uses the lower bound. In addition, the smooth sensitivity-based mechanism achieves the noise of which standard deviation is lower than $7$ for $\epsilon \ge 0.7$ for each datasets. The reason why we got these results is our approach depends only on the number of outliers, not on the number of dimensions. From these results, we can conclude that our framework is sufficiently practical in this setting.

\subsection{Top-$h$ Subspace Discovery}
\label{subsec:experiment-scenario2}
The experiments of top-$h$ subspace discovery shown in this subsection follow Scenario 2 of Section \ref{subsec:sinario}. The analyst investigates the subspace contains more outliers using query $q_{subspace}$. In these experiments, the dimensionality of the subspace is set as $1$; the analyst tries to detect $2$ out of $10$ subspaces by top-$h$ Subspace discovery.

For evaluation purposes, we partitioned the subspace into two classes: one is ``true'', indicating the subspace containing outliers; the other is ``false''. The utility of the results is measured from the precision and recall. The precision is evaluated by $precision = \frac{TP}{TP + FP}$, where $TP$ and $FP$ respectively denote true positive and false positive. The recall is evaluated by $recall = \frac{TP}{TP + FN}$, where $FN$ denotes false negative.
The prediction and recall are one thousand times average.
Privacy parameter was varied from $\epsilon=0.2$ to $3.2$.

Fig.~\ref{synthetic2}-\ref{ionosphere} (left) and Fig.~\ref{synthetic2}-\ref{ionosphere} (right) respectively represent the precision and recall, with changing $h$, the number of subspaces detecting. The precision decreases as $h$ grows, as shown in Fig.~\ref{synthetic2}. The recall can be improved with larger $h$ because the probability with which true subspaces are chosen increases. Because of sequential decomposition, the outputs of the exponential mechanism become noisy as $h$ increases. Therefore, the recall can be decreased if the effect of noise is dominant. As Fig.~\ref{synthetic2} shows, the effect of sequential composition was more dominant and smaller $h$ achieved larger recall in this experiments.
However there isn't distinctive subspace that has many outliers. It is difficult to apply top-$h$ subspace discovery when the difference of the number of ouliers are not.

For practical use, the precision and recall are preferred to be much higher than $1/2$. If the number of true subspaces can be known by analysts in advance, then $h$ should be set as small as possible. Privacy parameter $\epsilon$ and utility (precision and recall) share a tradeoff relation. Noting that the objective of outlier analysis is fundamentally conflicting with privacy protection, the choice of larger $\epsilon$, such as $0.8 \leq \epsilon \leq 1.6$, might be allowed.

\begin{figure}[t]
 \begin{minipage}[t]{.5\linewidth}
  \centering
  \includegraphics[width=75mm]{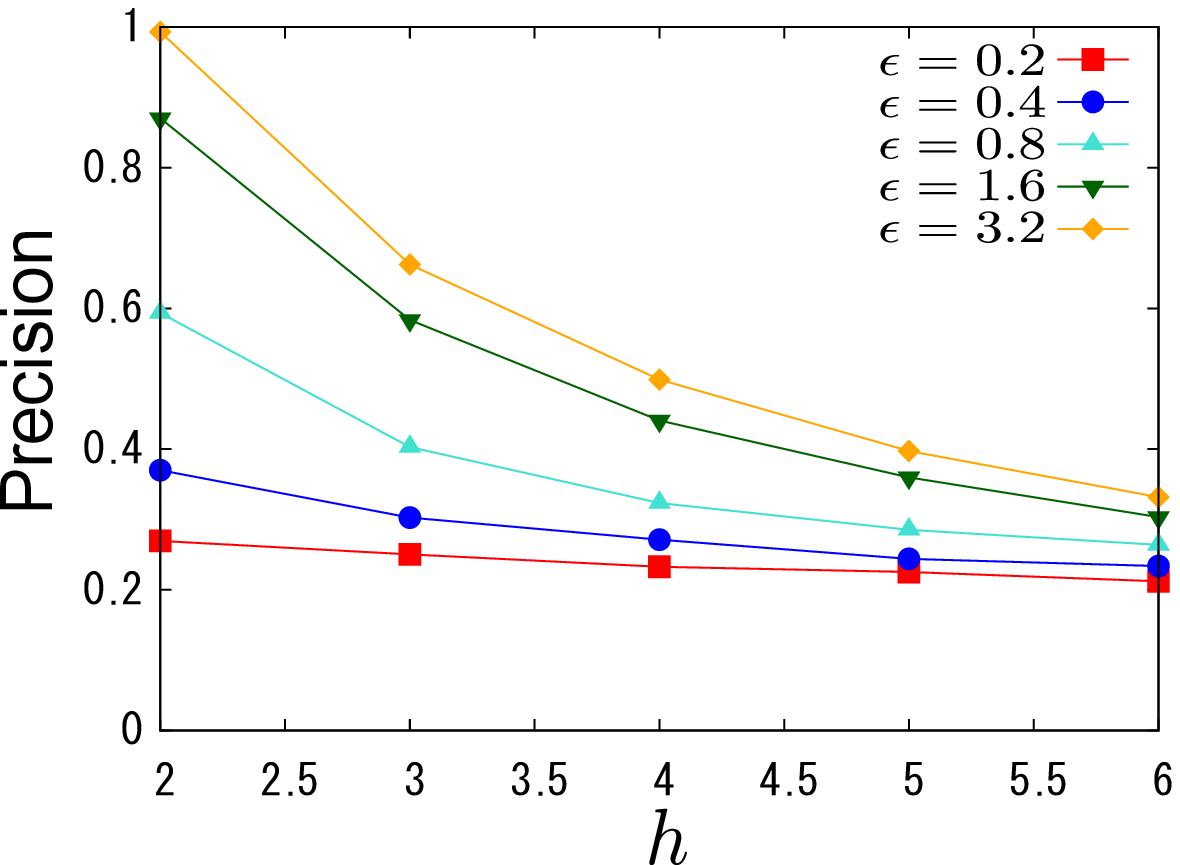}
 \end{minipage}
 \begin{minipage}[t]{.5\linewidth}
  \centering
  \includegraphics[width=75mm]{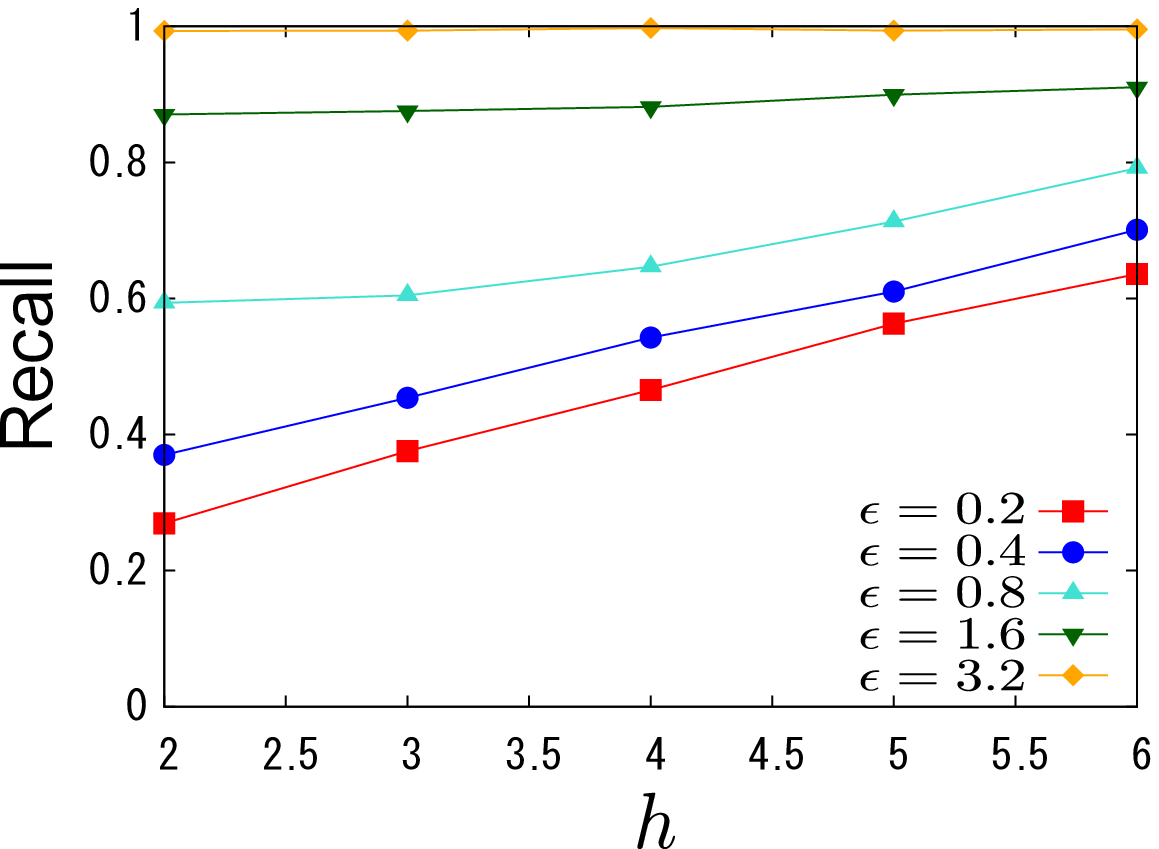}
 \end{minipage}
 \caption{The result of Synthetic 2 on Scenario 2}
  \label{synthetic2}
\end{figure}

\begin{figure}[t]
 \begin{minipage}[t]{.5\linewidth}
  \centering
  \includegraphics[width=75mm]{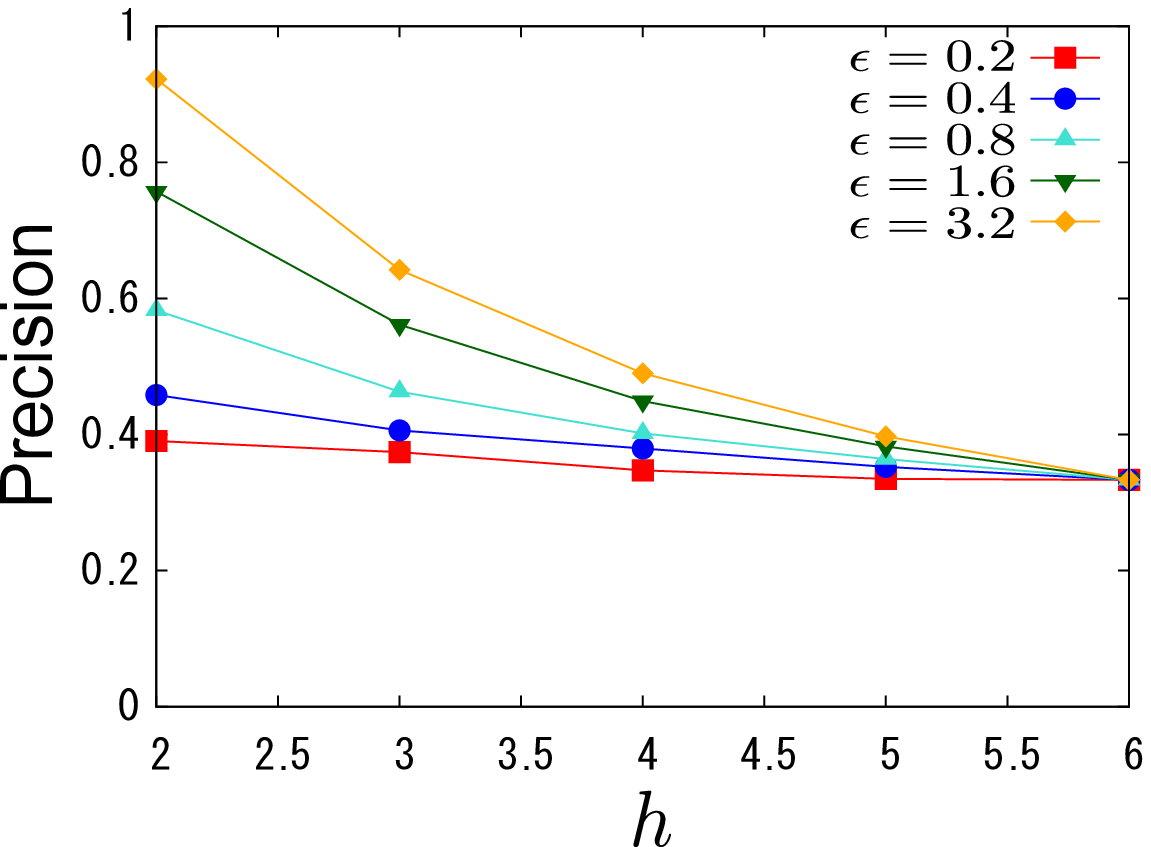}
 \end{minipage}
 \begin{minipage}[t]{.5\linewidth}
  \centering
  \includegraphics[width=75mm]{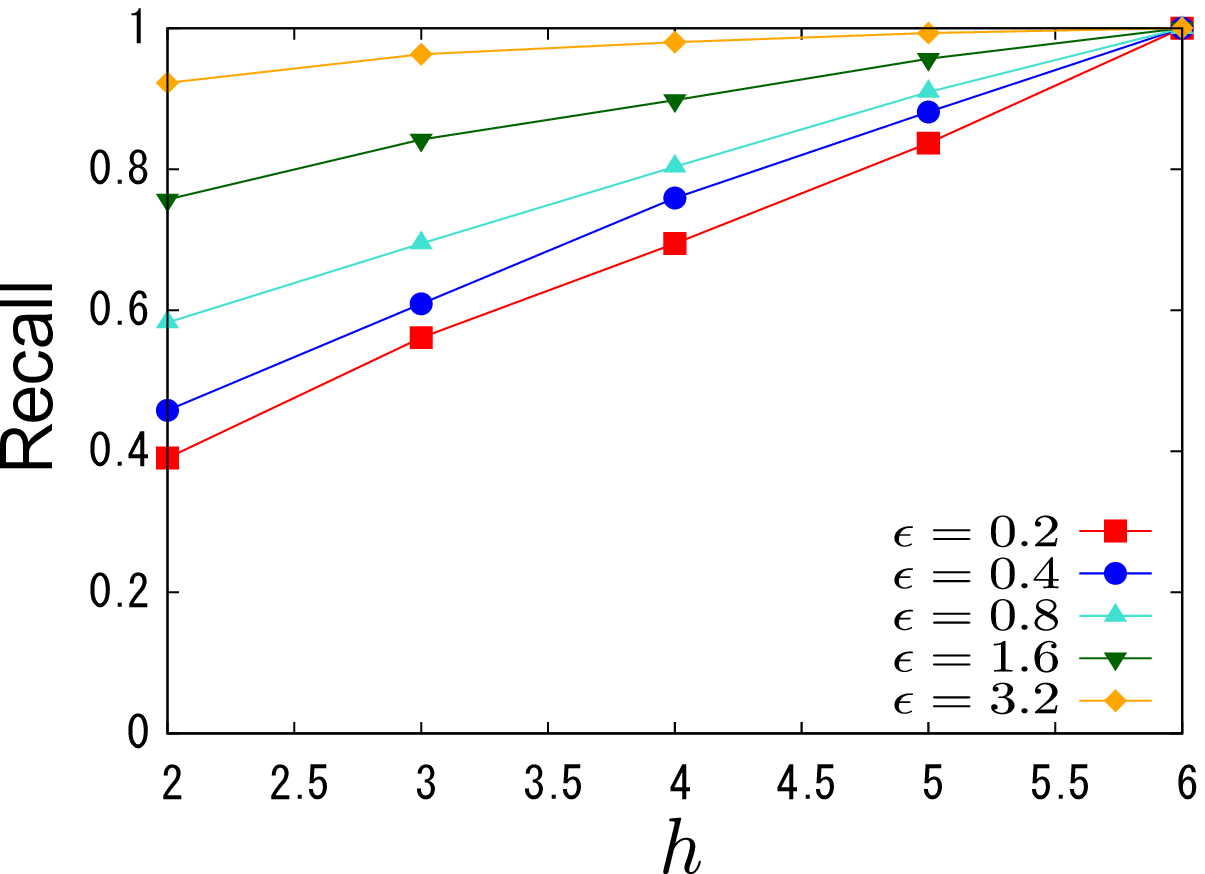}
 \end{minipage}
 \caption{The result of Adult 1 on Scenario 2}
 \label{adult2}
\end{figure}

\begin{figure}[t]
 \begin{minipage}[t]{.5\linewidth}
  \centering
  \includegraphics[width=75mm]{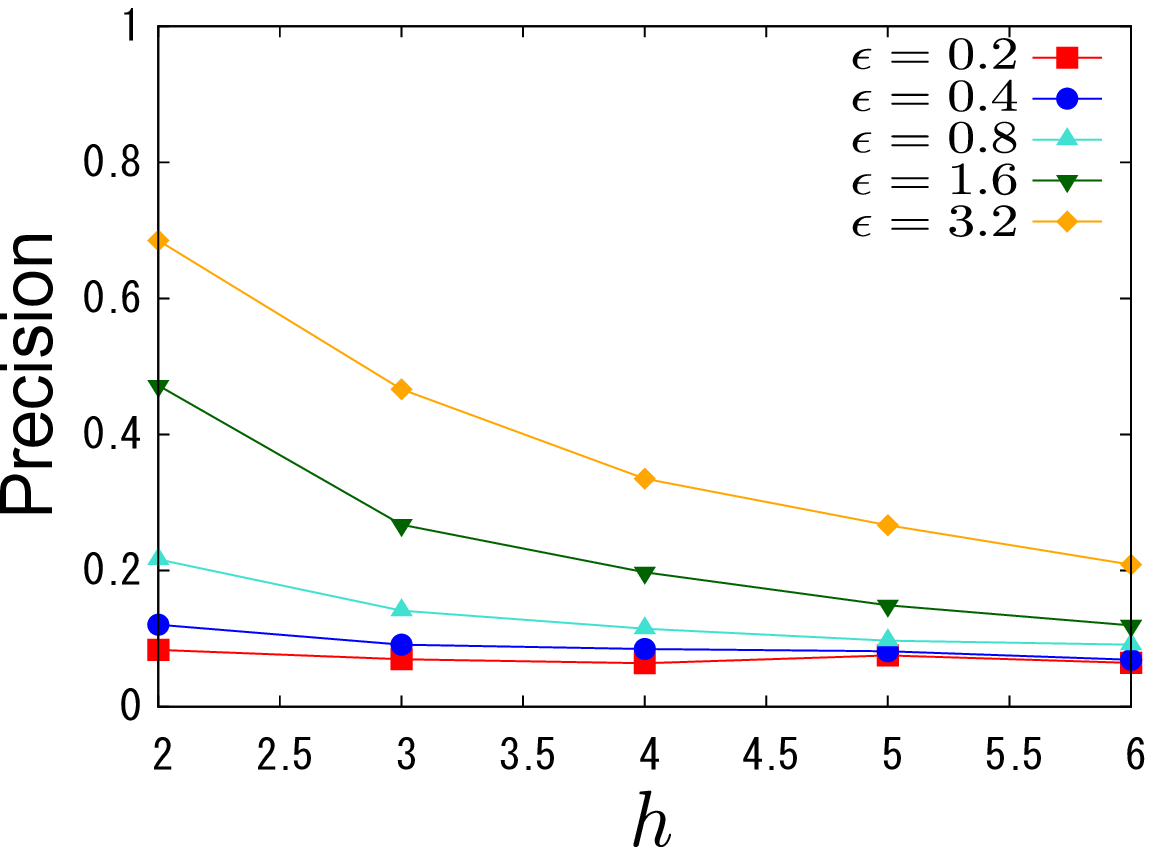}
 \end{minipage}
 \begin{minipage}[t]{.5\linewidth}
  \centering
  \includegraphics[width=75mm]{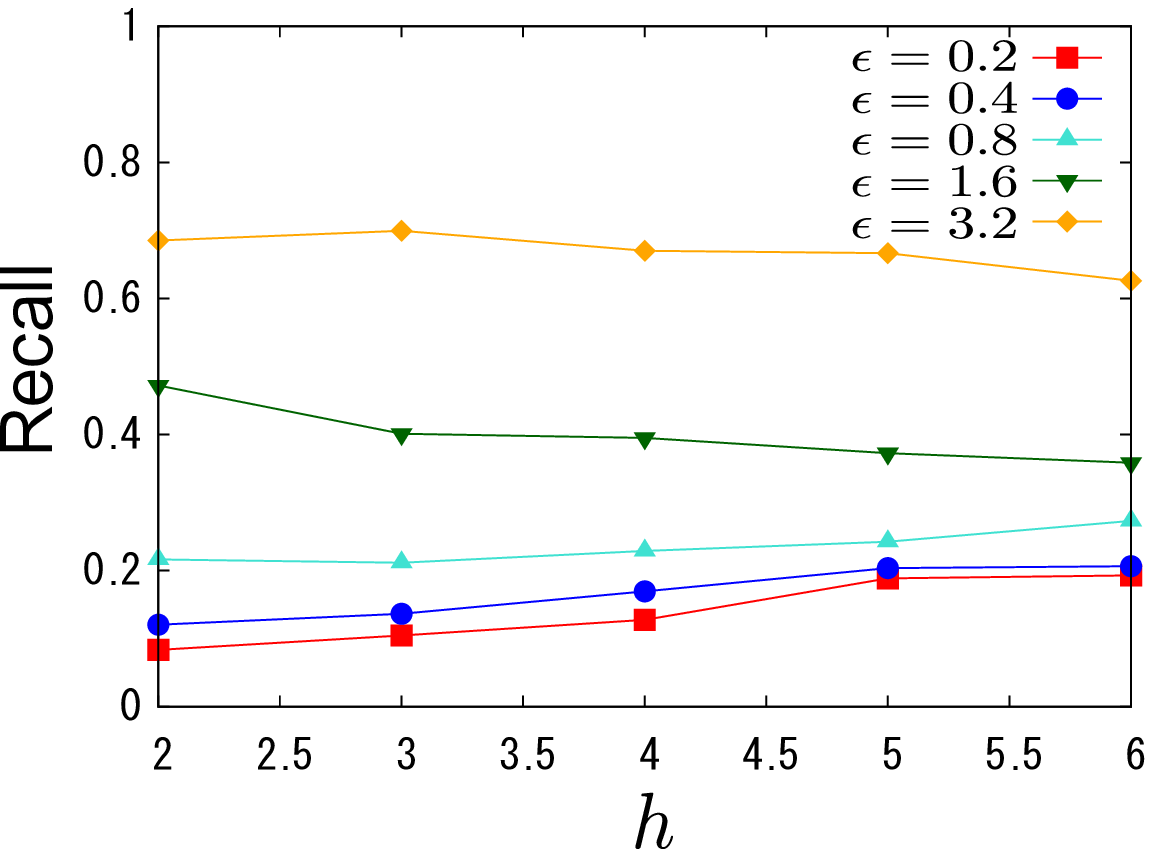}
 \end{minipage}
 \caption{The result of Ionosphere on Scenario 2}
 \label{ionosphere}
\end{figure}

\section{Conclusion}
 In this paper, we present the differentially private distance-based outlier analysis that consists of two different types of queries, the differentially private counting of outliers in given subspace and the differentially private discovery of subspaces.

For the query of counting of outliers, taking advantage of the smooth sensitivity~\cite{Nissim:2007:SSS:1250790.1250803}, the resulting output of the mechanism can be less noisy than that of the global sensitivity based mechanism. Although the evaluation of the smooth upper bound is often costly, we provide an efficient algorithm for evaluation of the smooth upper bound for the problem for outlier counting. This paper describes an initial step towards differentially private outlier analysis, and the experimental evaluation is performed with relatively small-size datasets. In our algorithm, we invoke the smallest enclosing ball algorithm that takes as input the power set of instances. Because of this construction, we need a more efficient algorithm for application to larger size datasets. 

For the query of discovery of subspaces, we employ the exponential mechanism and specifically design a utility function. Even though the variation of subspaces can grow exponentially in the data dimensionality, the proposed mechanism achieves better detection accuracy for high dimensionality.

 \bibliographystyle{unsrt}
 \bibliography{reference}

\end{document}